\documentclass[10pt,twocolumn,letterpaper]{article}

\usepackage{iccv}
\usepackage{times}
\usepackage{graphicx}
\usepackage{amsmath}
\usepackage{amssymb}
\usepackage{booktabs}
\usepackage{url}            
\usepackage{booktabs}       
\usepackage{amsfonts}       
\usepackage{amsmath}
\usepackage{amssymb}
\usepackage{mathtools}
\usepackage{amsthm}
\usepackage{soul}
\usepackage{nicefrac}  
\usepackage{diagbox}
\usepackage{paralist}
\usepackage{wrapfig}
\usepackage{lipsum}
\usepackage[toc,page]{appendix}

\usepackage{microtype}      
\usepackage{array}
\usepackage{multirow}
\usepackage{algorithm}
\usepackage{algorithmic}
\usepackage{eqparbox}
\theoremstyle{plain}
\newtheorem{theorem}{Theorem}[]

\theoremstyle{definition}

\theoremstyle{remark}

\usepackage{graphicx, caption, subcaption}

\usepackage[pagebackref=true,breaklinks=true,letterpaper=true,colorlinks,bookmarks=false]{hyperref}

\iccvfinalcopy 

\def\iccvPaperID{5335} 

\ificcvfinal\fi

\newcommand{\affmark}[1][*]{\textsuperscript{#1}}

\begin{document}

\title{\vspace{-20pt}Beyond Empirical Risk Minimization: Local Structure Preserving Regularization for Improving Adversarial Robustness}

\author{Wei Wei\affmark[1],  Jiahuan Zhou\affmark[2], Ying Wu\affmark[1]\\
\affmark[1]Department of Electrical and Computer Engineering, Northwestern University, IL, USA\\
\affmark[2]Wangxuan Institute of Computer Technology, Peking University, Beijing, China\\
{\tt\small wwzjer@u.northwestern.edu, jiahuanzhou@pku.edu.cn, yingwu@northwestern.edu}
}

\maketitle

\begin{abstract}
It is broadly known that deep neural networks are susceptible to being fooled by adversarial examples with perturbations imperceptible by humans. Various defenses have been proposed to improve adversarial robustness, among which adversarial training methods are most effective. However, most of these methods treat the training samples independently and demand a tremendous amount of samples to train a robust network, while ignoring the latent structural information among these samples. In this work, we propose a novel Local Structure Preserving (LSP) regularization, which aims to preserve the local structure of the input space in the learned embedding space. In this manner, the attacking effect of adversarial samples lying in the vicinity of clean samples can be alleviated. We show strong empirical evidence that with or without adversarial training, our method consistently improves the performance of adversarial robustness on several image classification datasets compared to the baselines and some state-of-the-art approaches, thus providing promising direction for future research.
\end{abstract}


\section{Introduction}
\label{sec_1intro}

Deep neural networks (DNNs) trained under the Empirical Risk Minimization (ERM) paradigm have been tremendously successful on plenty of computer vision tasks like image classification \cite{lecun2015deep}. However, it has been shown that DNNs are often fragile to adversarial examples which are crafted by adding an imperceptible perturbation on the natural input image to yield incorrect network predictions \cite{biggio2013evasion,Szegedy2014Intriguing,Goodfellow2015Explaining}. Such phenomena exist as a severe threat to the applicability of DNNs, especially when deployed in safety-critical applications like autonomous driving.

Many pieces of work have been devoted to explaining the existence of adversarial examples~\cite{Goodfellow2015Explaining,shamir2019simple,ilyas2019adversarial}. Nevertheless, the community has not reached a consensus. One persuasive argument is that ERM forces the DNNs to memorize the training data thus fragile when evaluated on the adversarial examples which differ from the training distribution~\cite{sehwag2019analyzing}. To address the issue of vulnerability against adversarial examples, currently, the most effective defense strategy is adversarial training, which is a min-max optimization process to augment the DNNs with the generated adversarial examples on the fly during training \cite{Kurakin2017Adversarial,madry2018towards,zhang2019theoretically}. Yet, the methods of adversarial training are still framed under the scope of the ERM, \ie, treating each example independently and emphasizing memorizing the adversarial examples.

\begin{figure}[t]
    \centering
    \includegraphics[width=0.95\linewidth]{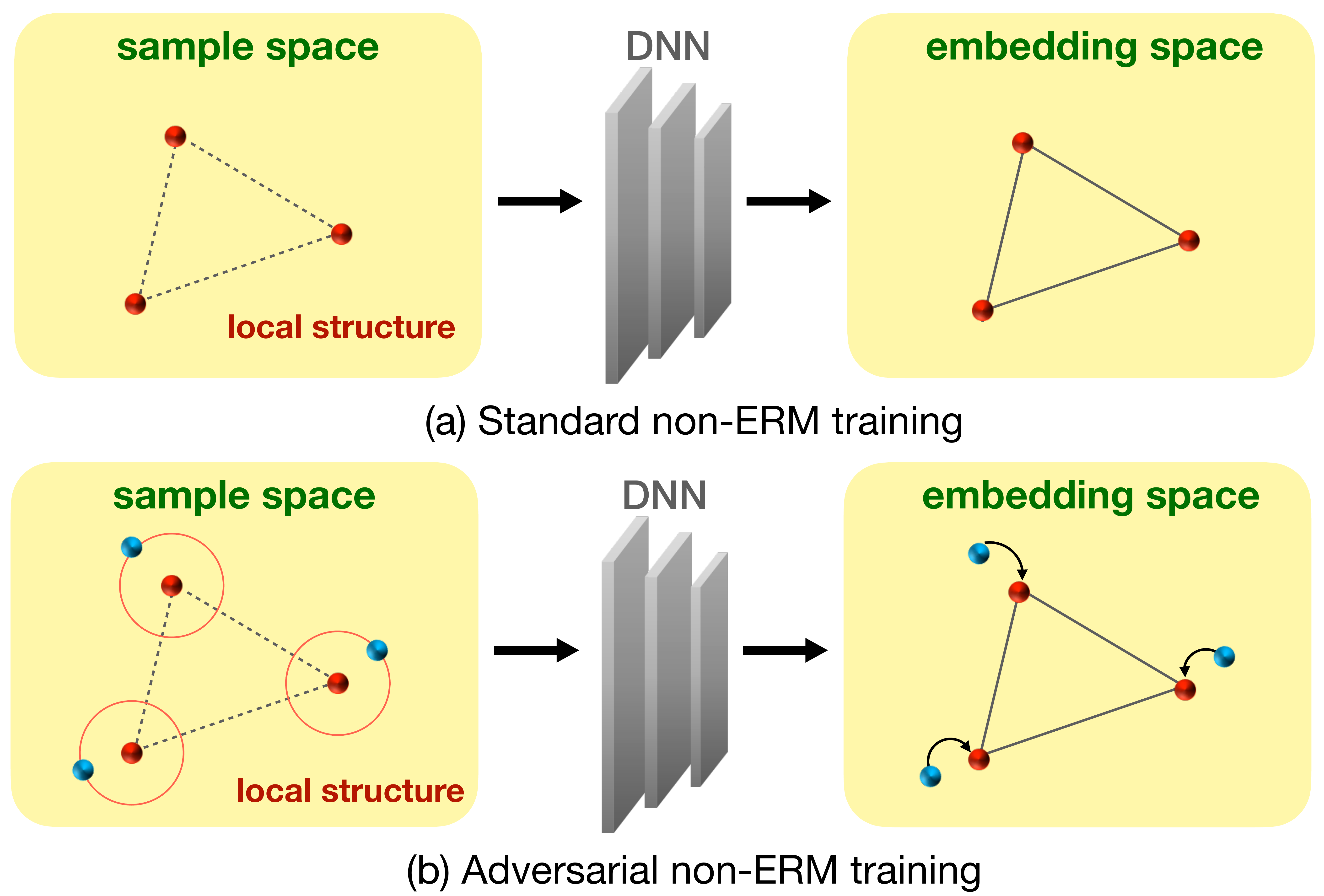}
    \caption{A schematic for Local Structure Preserving. For (a) standard training, the local structure in the sample space is constrained to be preserved in the embedding space. For (b) adversarial training, the adversarial example (\textcolor{blue}{blue} dot) not only needs to be mapped to the embedding of the corresponding natural example (\textcolor{red}{red} dot) but also the local structure on natural examples should not be destroyed.}
    \label{fig:fig1}
\end{figure}

Despite the significant improvement in adversarial robustness, one of the main challenges of adversarial training is \textbf{data insufficiency}. It has been theoretically verified that the amount of samples required to train a robust model is remarkably more enormous than that of standard training, especially for high-dimensional images~\cite{schmidt2018adversarially}. This explains the large gap between robust accuracy and standard accuracy when training with datasets with only limited samples.

Essentially, adversarial training under the ERM framework treats each sample independently and derives the loss function in a sample-wise manner, making it require a massive amount of training samples to infer the complicated data distribution of adversarial examples. However, if we could utilize the underlying data structure as prior knowledge, the strong demand for massive training samples could be alleviated. This motivates us to handle the data insufficiency challenge by exploring the internal structure that implicitly existed among training samples and investigating whether explicitly modeling them can benefit adversarial robustness. 

Considering the proximity of the adversarial example and its corresponding natural example in the sample space, we argue that the ideally robust DNNs should be capable to map the nearby examples in the sample space to the embeddings close to each other. For each sample, its nearest neighbors in the training dataset, which are supposed to lie in the vicinity of this sample on the local manifold, naturally provide useful local structure information of the input space. If the training samples within a neighborhood are mapped to be close to each other in the output embedding space, the learned networks could be more regular and reduce the possibility to map nearby adversarial examples to uncertain embeddings which bring out misclassification.

With this insight, our idea is to \textbf{preserve the local structure of the input samples in the learned embedding space}. More specifically, the sample proximity of the local manifold in the input space should not be violated during the feedforward process of the deep networks. We propose an extremely concise regularization term, named \textbf{Local Structure Preserving}, to force the proximity of the local structure of both input space and output embedding space to be as close as possible. Note that the metric of input space, which is usually unknown, can be naturally obtained through a pretrained deep network that functions as a feature extractor. Through this simple regularization, our learning process treats each training sample as \emph{non-i.i.d} and utilizes the latent interdependent relationship among them. As a consequence, it is not likely to map close samples in the input space to be farther away in the embedding space, which eases the vulnerability against adversarial attacks.

Moreover, our novel structural regularization can be simply incorporated into the adversarial training framework to further boost adversarial robustness. For a clean training sample, the local structure should not be altered significantly before and after it is adversarially perturbed. That means, the generated adversarial sample during adversarial training, should be mapped to similar positions in the embedding space as the corresponding clean one. Therefore, the output embedding of the adversarial example is not only close to the corresponding natural embedding but should be located within the local structure formed by its neighbors in the input space.

As a whole, our work mainly emphasizes the consideration of the structure that latently existed in the training samples, and thus is beyond the scope of ERM training. In this manner, the representation of the natural example or adversarial example is not independent for each sample, but follows a structure-wise alignment on each local manifold of the underlying data distribution. Thus the network prediction can be correlated by the local neighbors in the manifold. The idea of our non-ERM standard/adversarial training is depicted in Fig.~\ref{fig:fig1}.

In a nutshell, our contributions are three-fold:
\begin{compactenum}
    \item We propose a local structure preserving regularization for improving adversarial robustness. The structural information among samples is rarely investigated or used in previous adversarial learning methods.
    \item By explicitly modeling the interdependency among training data, the proposed method alleviates the notorious reliance on the massive amount of data to achieve adversarial robustness.
    \item We relate the proposed novel term with the goal of restricting the Lipschitz constant, which is known to be important for robustness. We report strong empirical results on several commonly used image classification datasets, demonstrating the improvement of adversarial robustness by our method. 
\end{compactenum}


\section{Related Work} 
\label{sec2_related}

\subsection{Adversarial Attacks and Defenses}
\label{sec2:adv}
Since the seminal works of~\cite{biggio2013evasion,Szegedy2014Intriguing} first discovered the vulnerability of deep neural networks and introduced the notion of adversarial examples that can attack well-trained networks to cause misclassification on these examples, a plethora of studies on generating aggressive adversarial attacks and establishing strong adversarial defenses have evolved. Early methods used the strategy of gradient obfuscation to defend against adversarial examples by intentionally masking the gradients since many attackers require the gradient information of the networks. Examples include defensive distillation~\cite{papernot2016distillation}, gradient shattering~\cite{buckman2018thermometer,guo2018countering}, stochastic gradients~\cite{Dhillon2018stochastic,xie2018mitigating},  vanishing/exploding gradients~\cite{song2018pixeldefend,samangouei2018defensegan}, \etc. However, it has been shown later that this strategy can be circumvented by stronger attacks \cite{carlini2017towards,athalye2018obfuscated,athalye2018synthesizing}.

In contrast, one of the most successful defensive strategies is adversarial training, which substitutes the clean training samples with the generated adversarial examples during training to improve the robustness of the deep networks. This was first introduced by \cite{Goodfellow2015Explaining}, where they used the adversarial examples generated by the Fast Gradient Sign Method (FGSM) to replace the original clean ones for training, and then generalized to be applied on large-scale datasets by \cite{Kurakin2017Adversarial}. However, the robustness against the FGSM attacker seems questionable to other attacks, as shown in \cite{tramer2018ensemble}, where they augmented the adversarial training set with the ensemble of adversarial examples crafted from several similar classifiers. Later, \cite{madry2018towards} used Projected Gradient Descent (PGD) adversarial examples, which they called the ``ultimate'' first-order adversary, in adversarial training and improved robustness to a wide range of adversarial attacks. Recently, multiple variants of adversarial training have been proposed. For example, TRADES~\cite{zhang2019theoretically} decomposed the adversarial training objective into two terms which control the natural accuracy and robustness as a trade-off. MART~\cite{Wang2020Improving} further incorporated an explicit differentiation of clean misclassified examples in the trade-off objective.

As claimed before, the adversarial training methods require tremendously many samples to learn a robust model~\cite{schmidt2018adversarially}. Some researchers augment millions of external data~\cite{torralba200880} into training ~\cite{gowal2020uncovering,carmon2019unlabeled,rebuffi2021data} to achieve the state-of-the-art performance. But the new challenge emerges as the extremely high requirement for computational resources. For example, to accomplish adversarial training in a reasonable time, Cloud TPUs are required. Even worse, the amount of data is limited itself for certain tasks, like medical image analysis. This motivates us to study from another angle -- to utilize the data structural relationship in a non-ERM manner.

\subsection{Non-ERM Learning Methods}
\label{sec2: non-erm}

Under the non-ERM training framework, our method treats each training sample as \emph{non-i.i.d.} and considers the structural modeling of the training samples through the idea that each sample can be influenced by its neighboring samples. A similar insight is shared in a series of data augmentation methods. For example, the method Mixup~\cite{zhang2018mixup} trained the networks on convex combinations of the pairs of examples and their labels. Following this, CutMix~\cite{yun2019cutmix}, AugMix~\cite{hendrycks2020augmix}, \etc, were proposed to augment more diverse mixup combinations of the training samples. These methods have been shown to benefit adversarial robustness. However, their methodology is significantly different from ours since their training pairs are randomly chosen and no structures are purposefully mined.

\section{Method} \label{sec3_model}


In this section, we present the proposed idea of Local Structure Preserving (\textsc{LSP}). First, in Sec.~\ref{sec3.1} and \ref{sec3.2}, we introduce our method in the settings without and with adversarial training. At the end of this section, we discuss the theoretical motivation of our idea.

\subsection{Local Structure Preserving} 
\label{sec3.1}

Suppose a natural example $x$ lies in a metric space $(\mathcal{X},||\cdot||)$, where $\mathcal{X}$ denotes the sample space and $||\cdot||$ denotes the metric. The input example $x$ is associated with a semantic class label $y\in \mathcal{Y}$, where $\mathcal{Y}$ denotes the label space. The goal is to learn a robust DNN classifier $f$ that maps from the sample space $\mathcal{X}$ to the label space $\mathcal{Y}$, by $N$ training pairs $\{(x_i,y_i)\}_{i=1}^N$ randomly sampled from the joint space $(\mathcal{X}, \mathcal{Y})$. The robustness hereby is toward adversarial examples $x^\prime$ which is crafted within the $\delta$-neighborhood (defined by $L_p$ norm $||\cdot||_p$) of the input $x$.

Next, we introduce our local structure preserving regularization. Suppose for an input example $x$, its nearest neighbors in the training dataset are $\{x_{n_1}, x_{n_2},..., x_{n_m}\}$, where $m$ denotes the area of neighborhood. We call $x$ as an anchor point and derive the differences of its neighbors to the anchor as:
\begin{equation}
    p_x^i = \frac{||x-x_{n_i}||}{\sum_{j=1}^m ||x-x_{n_j}||}, \quad i=1,2,..,m.
    \label{eq:pi}
\end{equation}
Here, the normalization in the denominator is applied to obtain the probability vector $P_x = [p_x^1, p_x^2, ..., p_x^m]$. This vector explicitly encodes the local structure information of each input example in the sample space. 

After the feedforward process of deep network $f$, the input $x$ and its neighbors are mapped to the output embedding space for classification, \ie, the layer output before softmax of the network. Again, we can still obtain the vector conveying the differences of $f(x)$ and $\{f(x_{n_1}), f(x_{n_2}),..., f(x_{n_m})\}$ in a similar manner as Eqn.~\ref{eq:pi}:
\begin{equation}
    q_x^i = \frac{||f(x)-f(x_{n_i})||}{\sum_{j=1}^m ||f(x)-f(x_{n_j})||}, \quad i=1,2,..,m.
    \label{eq:qi}
\end{equation}
Accordingly, the probability vector $Q_x = [q_x^1, q_x^2, ..., q_x^m]$ indicates the local structure of the output embedding space around the anchor.  

\begin{figure}
    \centering
    \includegraphics[width=0.9\linewidth]{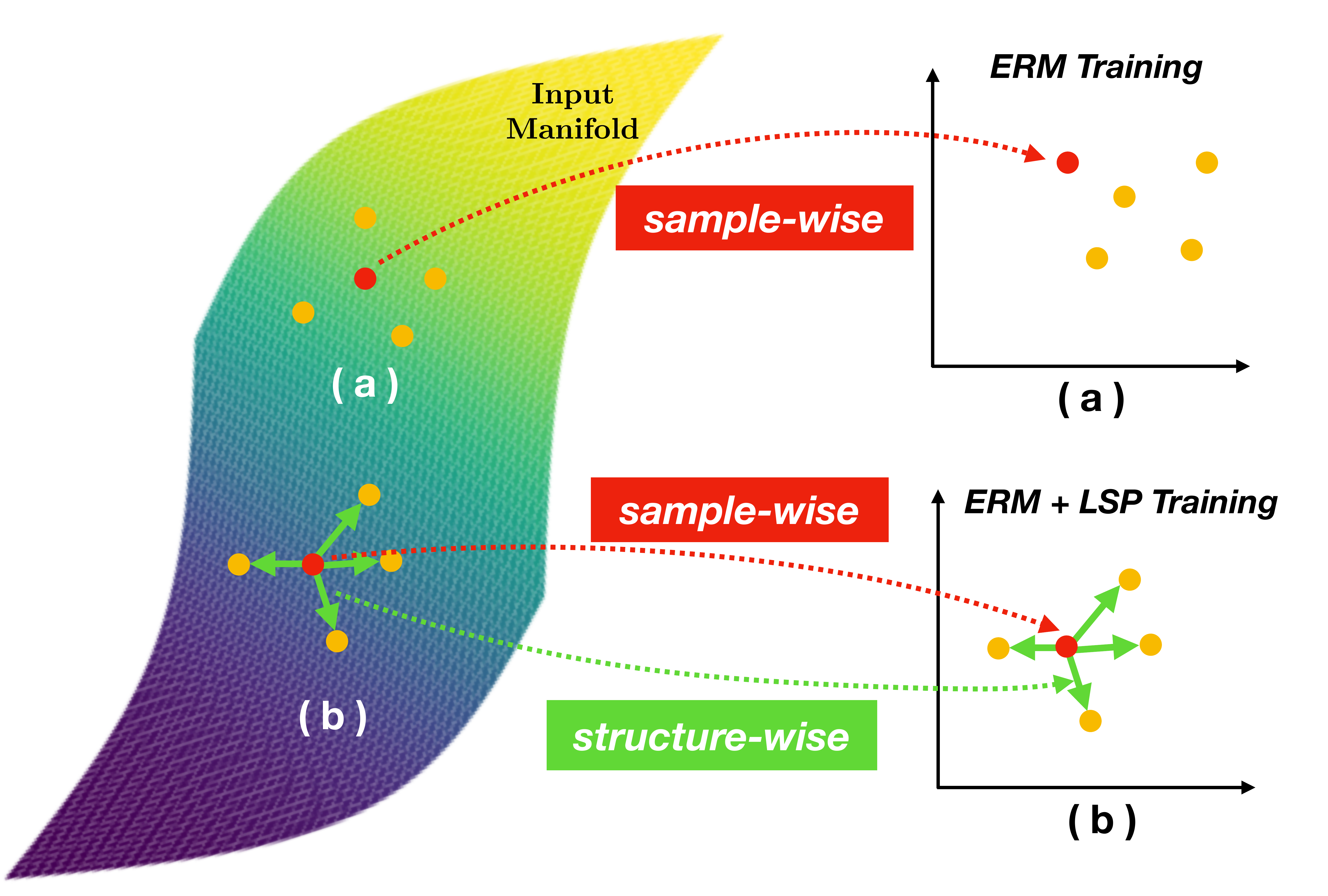}
    \vspace{-2mm}
    \caption{Left: the manifold represents the input space. Right: two coordinates represent the output embedding spaces. a) Only using sample-wise CE loss under the ERM training framework, the local structure of output space is not restricted and the order of local structure could be distorted. b) Regularizing ERM training with our proposed LSP, the local structure of input space is favored to be preserved in the output space.}
    \label{figure1}
\end{figure}

The goal of local structure preserving is to preserve these two structures, \ie, maintaining the proximity between closed examples of both spaces, expressed in the form of $P_x$ and $Q_x$. A straightforward way is to minimize the discrepancy as follows:
\begin{equation}
    \mathcal{L}_{\text{LSP}} = D(P_x,Q_x).
    \label{eq:dpq,non-adversarial}
\end{equation}
The scheme of defining discrepancy can be various. We discuss several terms, including KL-divergence, cosine similarity, and $L_1$/$L_2$ distance in Sec.~\ref{sec_ablation}. By introducing the regularization term in Eqn.~\ref{eq:dpq,non-adversarial}, we no longer treat each training sample as \emph{i.i.d.}, instead utilizing the local information among them.

However, a remaining unresolved problem is how we define the metric $||\cdot||$ in Eqn.~\ref{eq:pi} and \ref{eq:qi}. For Eqn.~\ref{eq:qi}, the metric is simply $L_2$ distance, since $f(x)$ encodes the feature representation and is normalized to follow an approximately Gaussian distribution. As for the metric in Eqn.~\ref{eq:pi}, since we have no clue about the distribution of the input space, one natural procedure is that we can use a pre-trained neural network as a feature extractor. Motivated by \cite{vangansbeke2020scan}, we use MoCo \cite{he2020momentum}, denoted by $g$, trained by contrastive loss in a self-supervised learning manner as the feature extractor to approximate the metric of input space, as follows:
\begin{equation}
    ||x_i - x_j|| \approx ||g(x_i) - g(x_j)||, \quad \forall x_i, x_j \in \mathcal{X}.
    \label{eqn:moco}
\end{equation}

Eqn.~\ref{eq:dpq,non-adversarial} serves as a regularization term in our method. Together with the traditional Cross Entropy loss:
\begin{equation}
    \mathcal{L}_{\text{CE}} = -\sum_{y\in \mathcal{Y}} y\log(f(x)_y),
    \label{eq:ce,non-adversarial}
\end{equation}
the total loss for training the network is as follows:
\begin{equation}
    \mathcal{L}_{\text{Total}} = \frac{1}{|\mathcal{X}|}\sum_{x\in\mathcal{X}} (\mathcal{L}_{\text{CE}} + \lambda \mathcal{L}_{\text{LSP}}),
    \label{eq:total,non-adversarial}
\end{equation}
where $\lambda$ is a balancing parameter of the two loss terms.

In Fig.~\ref{figure1}, we illustrate the insight behind our idea. Without the aid of local structure preserving regularization, training solely by Cross Entropy under the traditional ERM framework may lead to the result that the local structure of the output embedding space violates the original local structure in the input space, since no constraints are explicitly employed. Thus the adversarial examples which are also located within the neighborhood of the input sample may be mapped to the place far from the input in the output embedding space. By adding such an explicit structural prior constraint, the local structure around each input sample is preserved in the output embedding space. By regularizing each piece of the local manifold of the input space, the learned deep classifier is deemed to be more regular and potentially more resistant to adversarial attacks.

\begin{algorithm}[t]
   \caption{Local Structure Preserving}
   \label{algo}
\begin{algorithmic}[1]
   \STATE {\bfseries Input:} training data $\mathcal{D}=\{(x_i,y_i)\}_{i=1}^N$.
   \REPEAT
   \STATE Sample $(x,y)$ from $\mathcal{D}$. 
   \STATE Search $m$ nearest neighbors for $x$.
   \IF{standard training}
    \STATE Encode the local structure of $x$ in the input sample space by Eqn.~\ref{eq:pi}.
   \STATE Feed-forward to compute the loss in Eqn.~\ref{eq:total,non-adversarial}.
   \ELSIF{adversarial training}
    \STATE Encode the local structure of $x$ in the input sample space by Eqn.~\ref{eq:pi_adv}.
   \STATE Feed-forward to compute the loss in Eqn.~\ref{eq:total,adversarial}.
   \ENDIF
   \STATE Back-propagate to update the gradient of $f$.
   \UNTIL{training finished.}
\end{algorithmic}
\end{algorithm}

\subsection{LSP with Adversarial Training}
\label{sec3.2}

In Sec.~\ref{sec3.1} we introduced our method to improve the robustness of agnostic models. Here we note that our method can also be combined with adversarial training to further boost the robustness. 

During adversarial training, both the natural example $x$ and its adversarial example $x^\prime$ which is obtained by the inner maximization of the surrogate loss~\cite{madry2018towards,zhang2019theoretically} (\eg, Eqn.~\ref{eq:ce,non-adversarial}) are utilized. Different from Eqn.~\ref{eqn:moco} which requires a feature extractor $g$ beforehand, the distance in the natural input space can be approximated by their Euclidean distance in the representation space, 
\begin{equation}
    \tilde{p}_x^{i} \approx \frac{||f(x)-f(x_{n_i})||}{\sum_{j=1}^m ||f(x)-f(x_{n_j})||}, \quad i=1,2,..,m.
    \label{eq:pi_adv}
\end{equation}
We can readily search the nearest neighbors of each sample by maintaining a memory bank for storing the representations~\cite{wu2018unsupervised}. We initialize the memory bank with Gaussian random vectors and during each training iteration,  the representations in the memory bank will be updated. We will show in Supp. the fidelity of the local structure learned by mining the nearest neighbors from the memory bank which conveys faithful semantic relationship in the natural input space.

Next, we intend to preserve the local structure of the natural input space during adversarial training. Likewise, we can also model the local structure in the adversarial space:
\begin{equation}
    \tilde{q}_x^{i} = \frac{||f(x^\prime)-f(x_{n_i}^{\prime})||}{\sum_{j=1}^m ||f(x^\prime)-f(x_{n_j}^\prime)||}, \quad i=1,2,..,m.
    \label{eq:qi_adv}
\end{equation}

To enforce adversarial robustness, it is common to regularize the resemblance of the representation $f(x)$ of natural example and the representation $f(x^\prime)$ of adversarial example~\cite{zhang2019theoretically}. 
However, this still handles each sample independently without exploiting the local structure among data. As a comparison, we consider aligning the local structure from the natural and adversarial spaces, which regularizes the resemblance of the proximity of neighboring samples. Similar to Eqn.~\ref{eq:dpq,non-adversarial}, we minimize the discrepancy between:
\begin{equation}
    \tilde{\mathcal{L}}_{\text{LSP}} = D(\tilde{P}_x||\tilde{Q}_x),
    \label{eqn:lsp_adv}
\end{equation}
where $\tilde{P}_x = [\tilde{p}_x^1, \tilde{p}_x^2, ..., \tilde{p}_x^m]$,  $\tilde{Q}_x = [\tilde{q}_x^1, \tilde{q}_x^2, ..., \tilde{q}_x^m]$. By regularizing Eqn.~\ref{eqn:lsp_adv}, we aim to keep the distances of locally neighboring samples in the natural input space consistent in the corresponding adversarial space. By this regularization, we upgrade the alignment from a sample-wise manner to a local-region-wise manner. Thus, it is expected that the learned DNN will be more regular, as a consequence, to improve adversarial robustness. 

Combining with the common adversarial training loss~\cite{madry2018towards}:
\begin{equation}
    \tilde{\mathcal{L}}_{\text{CE}} = -\sum_{y\in \mathcal{Y}} y\log(f(x^\prime)_y),
    \label{eq:ce,adversarial}
\end{equation}
the total loss for adversarially training the network is:
\begin{equation}
    \tilde{\mathcal{L}}_{\text{Total}} = \frac{1}{|\mathcal{X}|}\sum_{x\in\mathcal{X}} (\tilde{\mathcal{L}}_{\text{CE}} + \lambda \tilde{\mathcal{L}}_{\text{LSP}}),
    \label{eq:total,adversarial}
\end{equation}
The pseudo-code of our method is in Algorithm \ref{algo}.

\subsection{Theoretical Analysis} \label{sec3.3}

In the previous work, \cite{Hein2017Formal,cisse2017parseval,cohen2019certified} studied to reduce the Lipschitz constant of the network $f$ with respect to an input example $x$, to improve the robustness. To highlight the theoretical motivation of LSP, we use the following off-the-shelf theorem to demonstrate that if the classifier is a Lipschitz function, the adversarial robustness can be guaranteed, and then analyze its relationship to LSP in an \emph{asymptotic} perspective. 

\begin{theorem} \label{theorem}
Suppose $a,b\in\mathcal{Y}$ are the most and second-most likely class prediction of the classifier $f$, i.e., 
\begin{equation}
    a = \arg\max_{l\in\mathcal{Y}} f(x)_l, \quad  
    b = \arg\max_{l\in\mathcal{Y}\backslash\{a\}} f(x)_l,
\end{equation}
and $p_a, p_b$ are the predictive probability of the classifier w.r.t. label $a$ and $b$, i.e.,
\begin{equation}
    p_a = f(x)_a, \quad p_b = f(x)_b.
\end{equation}
If function $f$ is locally $L_x$-Lipschitiz on input $x\in \mathcal{X}$, i.e., 
\begin{align}
    & \forall x\in X, \forall x^\prime, \text{s.t. } ||x^\prime-x||\le\delta, \quad \\
    &||f(x) - f(x^\prime)|| \le L_x ||x-x^\prime||.
    \label{eq:theorem}
\end{align}
Then $f$ is guaranteed to be robust at $x$ up to any perturbation of magnitude $\delta \le \frac{1}{2L_x}(p_a-p_b)$.
\end{theorem}


The proof will be given in Supp. Theorem \ref{theorem} implies that if the learned classifier satisfies the Lipschitz condition (Eqn.~\ref{eq:theorem}), certified robustness against the adversarial examples within the neighborhood of input $x$ can be realized. The upper bound of perturbation magnitude $\delta$ is influenced by two factors, $L_x$ and $(p_a-p_b)$. For the latter one, the explicit goal of classification is to make it larger, which indicates more confidence in the prediction of true label. As for $L_x$, from an \emph{asymptotic} view, if the training samples near the input $x$ are extremely dense, minimizing the LSP regularization in Eqn.~\ref{eq:dpq,non-adversarial} is equivalent to minimizing the discrepancy of $||f(x)-f(x^\prime)||$ and $||x-x^\prime||$ in Eqn.~\ref{eq:theorem}, thus $L_x$ is controllable as a constant. With constant denominator $L_x$ and larger numerator $(p_a-p_b)$ during training, the robustness radius $\delta$ of our model is expected to be larger.

Although such asymptotic analysis is ideal and can hardly be practical, the conducted experimental results in Sec.~\ref{sec_4exp} evidence the effectiveness of our method in real, thus our idea that treating each sample as \emph{non-i.i.d.} and preserving the local structure is verified to be effective.

\section{Experiments}
\label{sec_4exp}

In this section, we conduct extensive experiments to demonstrate the effectiveness of the proposed idea in boosting the performance of adversarial robustness against various adversarial attacks.

\subsection{Experimental Setups}
\label{sec:4.1}

\noindent
\textbf{Datasets:} We evaluate the performance of our method on four public image classification datasets: 
(1) CIFAR-10~\cite{krizhevsky2009learning}, which contains a training set of 50,000 examples and a testing set of 10,000 examples. Each example is a 32x32 color image associated with a label from 10 classes; 
(2) STL-10~\cite{coates2011analysis}, which has 10 classes, with 500 training images and 800 testing images per class. Each example is a 96x96 color image; 
(3) Street-View House Numbers (SVHN)~\cite{netzer2011reading}, which consists of 73,257 digits for training, and 26,032 digits for testing.  Each example is a 32x32 color image associated with a label from 10 classes; 
(4) Tiny-ImageNet~\cite{deng2009imagenet}, which contains 200 classes and each class has 500 images for training and 50 images for testing. Each example is a 64x64 color image.
We train with an NVidia RTX 2080 GPU with 24G memory for all datasets.


\noindent
\textbf{Evaluation Protocol:}
To evaluate the adversarial robustness, we use the three most common adversaries to attack all comparison methods, including (1) CW~\cite{carlini2017towards}: an optimization-based attacker that tries to find the minimal-distorted perturbation. The scheme is by minimizing the margin loss to generate adversarial examples that have minimal difference between the logit values of the most and second-most likely classes. (2) PGD~\cite{madry2018towards}: multi-step optimization with the random initialization, which was claimed to be the strongest first-order adversary. (3) AutoAttack (AA)~\cite{croce2020reliable}, currently the most reliable evaluation of adversarial robustness that is composed of an ensemble of four diverse attacks including three white-box attacks (APGD-CE~\cite{croce2020reliable}, APGD-DLR~\cite{croce2020reliable}, and FAB~\cite{croce2020minimally}) and one black-box attack (Square Attack~\cite{andriushchenko2020square}). Unless otherwise stated, we use $L_\infty$-norm bounded perturbations for all adversarial attacks. Further, we also report the clean accuracy where the model is under no attacks.

\begin{table*}
\begin{center}
{\scriptsize
\begin{tabular}{l cccc cccc cccc cccc}
 \toprule
 \multicolumn{1}{c}{\textbf{Methods}} & \multicolumn{4}{c}{CIFAR-10} & \multicolumn{4}{c}{STL-10} & \multicolumn{4}{c}{SVHN} & \multicolumn{4}{c}{Tiny-ImageNet} \\ \cmidrule(rl){2-5} \cmidrule(rl){6-9} \cmidrule(rl){10-13} \cmidrule(rl){14-17} 
 & Clean & CW & PGD & AA & Clean & CW & PGD & AA & Clean & CW & PGD & AA & Clean & CW & PGD & AA \\
 \midrule
 \textsc{Vanilla} & 94.56 & 3.68 & 6.26  & 2.44 & 78.97 & 48.12 & 19.58 & 17.52 & 96.55 & 58.03 & 47.80 & 40.60 & 65.37 & 25.85 & 2.56 & 1.62 \\ 
 \textsc{Mixup}  & 94.96 & 47.78 & 12.23 & 0.17 & \textbf{80.63}  & 46.95 & 14.31 & 8.76 & 96.65 & 62.74 & 23.87 & 6.69  & 66.51 & 28.05 & 4.83 & 1.19 \\
 \textsc{CutMix} & \textbf{95.87} & 20.73 & 13.14 & 0.10 & 78.13 & 48.56 & 19.88 & 16.88 & \textbf{97.29} & 42.88 & 26.95 & 1.34 & \textbf{68.94} & 24.77 & 2.58 & 0.45 \\ 
 \textsc{AugMix} & 94.46 & 16.33 & 16.17 & 9.08 & 76.12 & 49.68 & 24.06 & 20.83 & 96.87 & 45.71 & 44.93 & 36.39 & 65.52 & 32.55 & 5.11 & 3.17 \\ 
 \midrule
 \textsc{Lsp} & 94.22 & \textbf{62.53} & \textbf{28.00} & \textbf{15.11} & 76.85 & \textbf{52.49} & \textbf{25.69} & \textbf{23.69} & 96.78 & \textbf{78.86} & \textbf{52.92} & \textbf{45.74} & 67.58 & \textbf{36.08} & \textbf{7.79} & \textbf{5.28} \\ 
 \bottomrule
\end{tabular}
}
\vspace{-5pt}
\caption{The classification accuracy comparison of methods by standard training on different datasets. Numbers in bold indicate the best. Note that we achieve the strongest robustness against all attacks under all datasets.}
\label{table:non-adv-train}
\end{center}
\vspace{-20pt}
\end{table*}

\subsection{Comparison and Result Analysis}
\label{sec:4.2}
We conduct experiments under two settings -- w/wo adversarial training, in order to show our idea is beneficial to adversarial robustness \emph{per se}, and can further boost the performance of adversarial training. We denote our method for standard and adversarial training as \textsc{LSP} and \textsc{LSP+}.

\subsubsection{Standard Training}
\label{sec:4.2.1}

\noindent
\textbf{Comparison Methods:}
For the setting of standard training, \ie, adversarial examples are not generated during training, we compare with: 1) \textsc{Vanilla}, baseline classification networks by Cross Entropy loss; 2) \textsc{Mixup}~\cite{zhang2018mixup}, which is trained on convex combinations of pairs of samples and their labels; and its successors 3) \textsc{CutMix}~\cite{yun2019cutmix}; 4) \textsc{AugMix}~\cite{hendrycks2020augmix}. For the latter three methods, we compare with them since they also consider the structure among samples that share a similar spirit with our idea (but with totally different motivations and methodologies), and have provided evidence to improve adversarial robustness.

\noindent
\textbf{Experimental Settings:}\!
For CIFAR-10, STL-10 and SVHN, we use ResNet-18~\cite{he2016deep} as the backbone. For Tiny-ImageNet, we use Preact ResNet-18~\cite{he2016identity} as the backbone. The backbones are consistent for all comparison methods to ensure fairness. For all methods, we train 100 epochs with the optimizer of SGD, and the learning rate is initially set as 0.1, divided by 10 in the epoch 75 and 90. For the proposed local structure preserving term, we use MSE loss, \ie, minimizing the $L_2$ distance in Eqn.~\ref{eq:dpq,non-adversarial} and the balancing parameter $\lambda$ in Eqn.~\ref{eq:total,non-adversarial} is tuned to be 1. The number of nearest neighbors $m$ is 8. Ablation studies of these hyperparameters are provided in Sec.~\ref{sec_ablation} and Supp. The feature extractor is trained with the same data and backbone used in robustness evaluation. We select instance discrimination~\cite{wu2018unsupervised} as the pretext task.

For the attack parameters in evaluation, since not augmented with adversarial examples during training, it is hard to expect these methods can defend against adversarial attacks with the commonly used maximum perturbation budget $\delta=8/255$. Therefore, we seek to set the budget as a milder $2/255$ for all attacks. For the CW attack, the box-constraint parameter $c$ is 0.01. For the PGD attack, the attacking step is 10 with step size $1/255$.

\noindent
\textbf{Results Analysis:}
In Tab.~\ref{table:non-adv-train}, we show the results of all methods that are without adversarial training. Basically, \textsc{Mixup}, \textsc{CutMix} and \textsc{AugMix}, which are also in the scope of the non-ERM training framework, can improve the robustness over \textsc{Vanilla} under certain circumstances, but the improvement is not consistent. For example, \textsc{Mixup} decreases the robustness against PGD on STL-10 and SVHN, \textsc{CutMix} decreases the robustness against CW on SVHN and Tiny-ImageNet, and \textsc{AugMix} decreases the robustness against both CW and PGD on SVHN. Even worse, in consideration of the most reliable robustness evaluation by AA, \textsc{Mixup} and \textsc{CutMix} worsen the robustness on all datasets, while \textsc{AugMix} worsens the robustness on SVHN. 

\begin{figure}[t]
\begin{center}
    \begin{subfigure}[b]{0.48\textwidth}
    \centering
    \includegraphics[width=\linewidth]{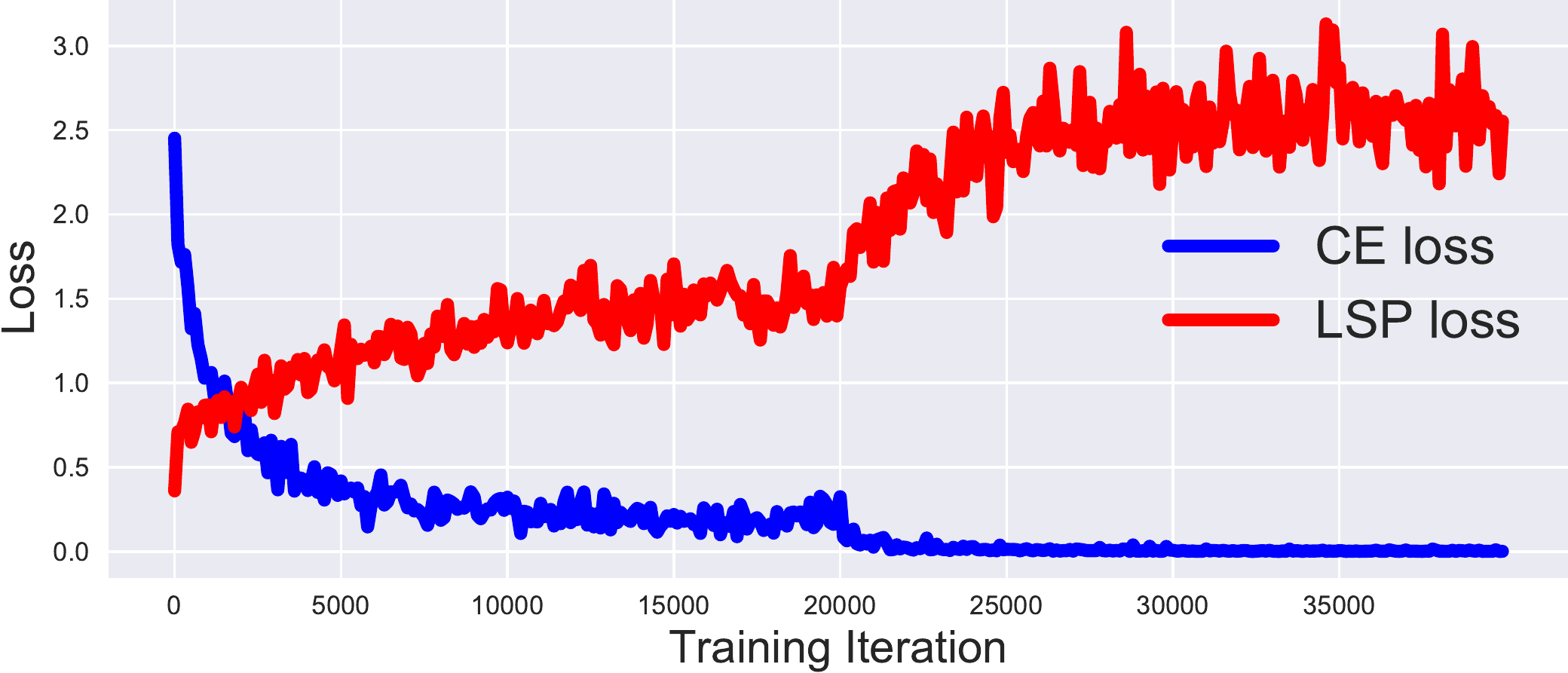}
    \end{subfigure}
    \begin{subfigure}[b]{0.48\textwidth}
    \includegraphics[width=\linewidth]{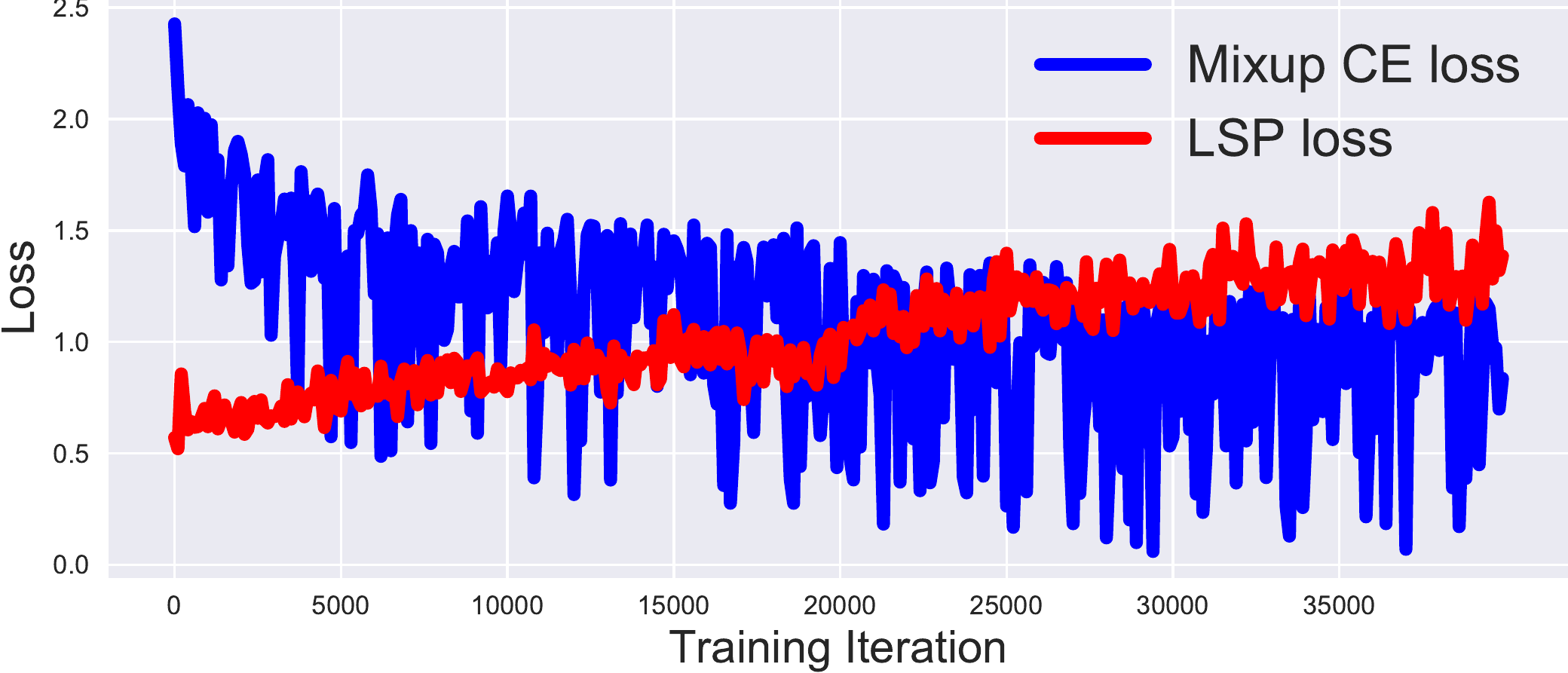}
    \end{subfigure}
    \vspace{-20pt}
    \caption{The training loss curves of \textsc{Vanilla} (top) and \textsc{Mixup} (bottom) on CIFAR-10. Model is trained without LSP regularization. For both methods, LSP loss increases during training.}
    \label{figure3}
\end{center}
\vspace{-20pt}
\end{figure}

As a comparison, our method consistently improves the robustness performance against all attacks over \textsc{Vanilla}, and shows superior robustness over all three non-ERM methods, just by adding the local structure preserving regularization term. This indicates that preserving the local structure in the original space during training can make the network training more regular and better repair the ``pockets'' in the data manifold~\cite{Szegedy2014Intriguing} which turns into adversarial examples. However, three non-ERM comparison methods only provide flimsy robustness, and we suppose the reason is that only restricting the data structure from the convex combination perspective is insufficient to provide reliable robustness due to the extreme non-linearity of high-dimensional data space. On the contrary, our operation on the data structure is more general, leading to stronger robustness. Also, note that the defensive effect of \textsc{LSP} is most prominent against CW attack, particularly for CIFAR-10. We analyze the reason is due to that CW attack generates adversarial examples by encouraging the logit value of the second-most likely class in the output embedding vector to exceed the one of the most likely class, while our method imposes prior information on the output embedding, thus not easy to be fooled.

In Fig.~\ref{figure3}, we depict the loss curves of two terms, Cross Entropy and LSP, for both \textsc{Vanilla} and \textsc{Mixup} model. The Cross Entropy losses of both models exhibit a decreasing trend as anticipated, while the LSP losses keep rising. For \textsc{Vanilla}, this indicates that ERM memorizes the training data while may destroy the interior data structure, thus leading to the poor robustness on adversarial examples which are slightly different distributed. While for \textsc{Mixup}, the LSP loss is relatively smaller but still increasing, which means the data structure is still undermined. Therefore, the LSP loss needs to be explicitly optimized and our performance shows its importance on adversarial robustness.

\subsubsection{Adversarial Training}

\begin{table*}[t]
\begin{center}
{\scriptsize
\begin{tabular}{l cccc cccc cccc cccc}
 \toprule
 \multicolumn{1}{c}{\textbf{Methods}} & \multicolumn{4}{c}{\textbf{CIFAR-10}} &
 \multicolumn{4}{c}{\textbf{CIFAR-100}} &
 \multicolumn{4}{c}{\textbf{SVHN}} & \multicolumn{4}{c}{\textbf{Tiny-ImageNet}} \\ \cmidrule(rl){2-5} \cmidrule(rl){6-9} \cmidrule(rl){10-13} \cmidrule(rl){14-17}
 & Clean & CW & PGD & AA & Clean & CW & PGD & AA & Clean & CW & PGD & AA & Clean & CW & PGD & AA \\
 \midrule
 \textsc{AT}       & 83.60 & 50.91 & 52.37 & 47.15 & 57.02 & 26.07 & 27.34 & 24.04 &92.02 & 54.51 & 57.19 & 49.50 & 48.79 & 21.09 & 23.38 & 16.92\\
 \textsc{TRADES}   & 81.77 & 50.96 & 52.86 & 48.99 & 56.33 & 24.47 & 28.62 & 23.62 & 89.69 & 56.30 & 60.11 & 53.15 & 50.01 & 18.33 & 22.93 & 16.90\\
 \textsc{MART}     & 78.51 & 50.26 & \textbf{55.29} & 47.79 & 51.15 & 25.67 & \textbf{29.36} & 24.10 & 89.83 & 54.89 & 61.22 & 49.64 & 44.35 & 19.62 & 23.48 & 17.01\\
 \textsc{HAT}      & \textbf{84.75} & 50.95 & 52.98 & 48.64 & 59.10 & 24.55 & 27.37 & 23.47 & 91.98 & \textbf{58.51} & 61.11 & 51.43 & 51.83 & 18.25 & 22.16 & 16.83\\
 \midrule
 \textsc{LSP+}     & 84.74 & \textbf{51.01} & 53.53 & \textbf{49.26} & \textbf{59.68} & \textbf{26.71} & 28.34 & \textbf{24.55} & \textbf{92.61} & 58.29 & \textbf{62.33} & \textbf{53.90} & \textbf{52.57} & \textbf{21.98} & \textbf{23.58} & \textbf{17.63}\\
 \bottomrule
\end{tabular}
\vspace{-5pt}
\caption{\label{table:adv-train}The classification accuracy comparison of methods by Adversarial Training. Numbers in bold indicate the best. Note that we achieve the strongest robustness against the most reliable robust evaluation attack AA under all datasets.}
}
\end{center}
\vspace{-20pt}
\end{table*}

\noindent
\textbf{Comparison Methods:} In this experiment, we compare our performance with various state-of-the-art adversarial training methods, including 1) \textsc{AT} \cite{madry2018towards}, which substitutes the original training samples with the PGD adversarial examples in training and is known for the resistance to a wide range of adversarial attacks; 2) \textsc{TRADES} \cite{zhang2019theoretically}, which regularizes to trade adversarial robustness off against clean accuracy, and we set trade-off parameter as 6 for optimal robustness following their paper; 3) MART~\cite{Wang2020Improving}, which extends TRADES by considering a regularized adversarial loss which incorporates an explicit differentiation of misclassified examples. The balancing parameter in MART is 5 following their paper; 4) HAT~\cite{rade2022reducing}, a recent state-of-the-art adversarial training method that achieves a better accuracy-robustness trade-off by reducing the excessive margin of clean samples. We exactly follow their reported hyperparameters.

\noindent
\textbf{Experimental Settings:}
Here we consider the four commonly evaluated dataset -- CIFAR-10/100, SVHN, and Tiny-ImageNet for adversarial robustness.
The backbone for each dataset is kept the same as in Sec.~\ref{sec:4.2.1}. We follow the training settings in~\cite{rade2022reducing}. The adversarial example generated during training is by PGD attack with perturbation budget $\delta=8/255$ and 10 steps. We adopt the early stopping scheme~\cite{rice2020overfitting} to evaluate the model with the best adversarial accuracy during training. For testing, CW attack refers to the PGD attack with CW loss with  $\delta$ as 8/255.

\noindent
\textbf{Results Analysis:} 
As shown in Tab.~\ref{table:adv-train}, when evaluating with the most reliable robustness metric AA, our method achieves better performance than all comparison methods on all datasets. Compared to our baseline method \textsc{AT}, all evaluation metrics are improved, where some are of them significant, like AA for SVHN. Notably, the clean accuracy of \textsc{LSP+} is also better or on par with all comparison methods. This again verifies the effectiveness of our method, where we show the plausibility of the proposed LSP regularization by consistent improvement over baseline without this regularization and competitive performance over other strong defensive methods.

\subsection{Evaluation with Larger Models}
\label{sec:benchmark}

\begin{table}[t]
\centering
{\small
\begin{tabular}{l c c rr}
\toprule 
 Dataset & Backbone & Method & Clean & AA \\
\midrule 
\multirow{9}*{CIFAR-10} &
\multirow{9}*{WRN-34-10}  
& PGD-AT~\cite{madry2018towards} &  87.14 & 44.04 \\
~& & YOPO~\cite{zhang2019you}  & 87.20 & 44.83 \\
~& & TLA~\cite{mao2019metric}  & 86.21 & 47.41 \\
~& & TRADES~\cite{zhang2019theoretically}  & 84.92 & 52.64 \\
~& & FAT~\cite{zhang2020attacks} & 84.52 & 53.51 \\
~& & BoT~\cite{pang2021bag} & 85.48 & 53.80 \\
~& & LBGAT~\cite{cui2021learnable} & 88.22 & 52.86 \\
~& & AWP~\cite{wu2020adversarial} & 85.36 & 56.17 \\
\cmidrule{3-5}
~& & LSP+ & \textbf{88.31} & \textbf{57.21} \\

\midrule


\multirow{5}*{CIFAR-100} &
\multirow{5}*{WRN-34-10} 
& TRADES~\cite{zhang2019theoretically} & 56.50 & 26.87 \\
~& & LBGAT~\cite{cui2021learnable} & 60.64 & 29.33 \\
~& & HAT~\cite{rade2022reducing} & 61.50 & 28.88 \\
\cmidrule{3-5}
~& & LSP+ & \textbf{62.31} & \textbf{28.99} \\


\bottomrule
\end{tabular}
}
\vspace{-5pt}
\caption{The evaluation of clean accuracy and AutoAttack robustness of our method in comparison with several state-of-the-art methods in RobustBench. Note that no additional real/synthetic data is used for all methods.}
\label{table:benchmark}
\vspace{-10pt}
\end{table}

\noindent 
\textbf{Experimental Setup:}
To further evaluate the performance of the proposed method, we compare our method with several state-of-the-art adversarial defensive methods in RobustBench~\cite{croce2020robustbench} using larger backbone networks. It is noteworthy that we do not aim to claim that our method is the strongest defensive method. As mentioned at the end of Sec.~\ref{sec2:adv}, the state-of-the-art performance in RobustBench relies tremendously on brute-force augmenting additional real or synthetic images (over million-level), \eg,~\cite{gowal2020uncovering,carmon2019unlabeled,rebuffi2021data}. On the contrary, the intention of our method is to alleviate the strong need for the massive amount of data, from the perspective of modeling interdependency structure among data. Thus we do not compare with these data-demanding approaches. For fairness, in Tab.~\ref{table:benchmark}, we compare with the state-of-the-art additional-data-free methods in RobustBench with the same backbone and the same amount of training data.
Specifically, for CIFAR-10/100, we use the backbone of WideResNet-34-10 (WRN-34-10)~\cite{Zagoruyko2016Wide} with a network depth of 34 and width of 10.

\noindent
\textbf{Results:} We show the benchmark results in Tab.~\ref{table:benchmark}. For CIFAR-10, under the network WRN-34-10 with larger capacity, we achieve both the best performance of clean accuracy and AutoAttack robustness, compared to a handful of competitive adversarial defensive methods, where many of them suffer from the trade-off between clean and robust accuracy. The same conclusion can be drawn from the experiment on CIFAR-100 with larger model.

\subsection{Ablation Study}
\label{sec_ablation}
Here we conduct various ablation experiments to deeply analyze our method.

 

\noindent
\textbf{Global or Local Structure Preserving?}
The first question is whether the global or local structure should be preserved. The global scheme preserves the structure of random examples globally sampled from the dataset. In Tab.~\ref{ablation:global or local}, the robustness performance of the global scheme is much worse than the local scheme in our method.


\noindent
\textbf{Local Structure Preserving Loss Term:}
An important thing to consider is the optimal form of local structure preserving loss term as a regularizer. Here we empirically study the difference of several forms of discrepancy in Eqn.~\ref{eq:dpq,non-adversarial}, including minimizing the KL-divergence, maximizing the cosine similarity, and minimizing the L1/L2 distance. The results are shown in Tab.~\ref{table:ablation:loss}. As shown, all choices of local structure preserving terms achieve some improvements in robustness against different attacks. Among these, minimizing L2 distance works the best.

\noindent
\textbf{Computation Overheads:} Our method brings limited
extra overhead. For example, on CIFAR-10, the average training
time per iteration for \textsc{LSP} and its baseline \textsc{Vanilla} are
0.165s and 0.114s, for \textsc{LSP+} and its baseline \textsc{AT} are 0.456s and 0.414s.

\begin{table}[t]
\begin{center}
{\small
\begin{tabular}{lccr}
\toprule
Which to Preserve & None & Global & Local \\
\midrule
Clean & \textbf{94.56} & 93.56 & 94.22\\
CW & 3.68 & 44.17 & \textbf{62.53}\\
$\text{PGD}$ & 6.26 & 21.26 & \textbf{28.00} \\
\bottomrule
\end{tabular}
}
\end{center}
\vspace{-15pt}
\caption{The performance comparison of preserving the global and local structure on CIFAR-10.}
\label{ablation:global or local}
\vspace{-10pt}
\end{table}

\begin{table}[t]
\begin{center}
{\small
\begin{tabular}{lccccr}
\toprule
Loss  & K-L & Cosine & L1 & L2  \\
\midrule
Clean & 94.38 & 92.67 & 86.42 & 93.65  \\
CW & 40.88 & 56.52 & 57.17 & \textbf{58.14} \\
$\text{PGD}$ & 5.35 & 13.60 & 13.60  & \textbf{15.87}  \\
\bottomrule
\end{tabular}
}
\end{center}
\vspace{-15pt}
\caption{Various choices of local structure preserving loss terms, testing without adversarial training on the CIFAR-10 dataset.}
\label{table:ablation:loss}
\vspace{-10pt}
\end{table}

Due to the limited space, we leave other important ablation analyses in Supp., which include Neighborhood Area, Loss Weighting Parameter, Different Perturbation Norms, and Evaluation with Limited Training Data, \etc.

\section{Conclusion} 
\label{sec_5con}

In this paper, we rethink the defect of the ERM learning framework for adversarial robustness. We argue that treating each training sample as \emph{i.i.d.} may lead to the violation of the local structure of the output embedding space against the original input space, thus leading to learning a distorted network mapping that has poor adversarial robustness. To solve this issue, we propose a novel yet simple idea to preserve the local structure from the input space to the output embedding space, thus making the network mapping regular. From our experimental observations, such a simple regularization significantly and consistently improves the metrics of adversarial robustness against various adversarial attacks on multiple datasets. The advantage of our method lies in alleviating the tremendous need for the massive amount of data by modeling structure prior. Therefore, when training with the same amount of data, our method is also competitive against the state-of-the-art adversarial training methods. 

Different from brute force introducing more training data as the solution for adversarial robustness,  we intend to solve from another way by emphasizing \emph{non-i.i.d.} structural prior modeling to alleviate the yearning for data. For future work, the more delicate structural modeling scheme is worth to be studied and deeper theoretical certificates are necessary. 


{\small
\bibliographystyle{ieee_fullname}
\bibliography{egbib}

\begin{thebibliography}{10}\itemsep=-1pt

\bibitem{andriushchenko2020square}
Maksym Andriushchenko, Francesco Croce, Nicolas Flammarion, and Matthias Hein.
\newblock Square attack: a query-efficient black-box adversarial attack via
  random search.
\newblock In {\em European Conference on Computer Vision}, pages 484--501,
  2020.

\bibitem{athalye2018obfuscated}
Anish Athalye, Nicholas Carlini, and David Wagner.
\newblock Obfuscated gradients give a false sense of security: Circumventing
  defenses to adversarial examples.
\newblock In {\em International Conference on Machine Learning}, pages
  274--283. PMLR, 2018.

\bibitem{athalye2018synthesizing}
Anish Athalye, Logan Engstrom, Andrew Ilyas, and Kevin Kwok.
\newblock Synthesizing robust adversarial examples.
\newblock In {\em International Conference on Machine Learning}, pages
  284--293. PMLR, 2018.

\bibitem{biggio2013evasion}
Battista Biggio, Igino Corona, Davide Maiorca, Blaine Nelson, Nedim
  {\v{S}}rndi{\'c}, Pavel Laskov, Giorgio Giacinto, and Fabio Roli.
\newblock Evasion attacks against machine learning at test time.
\newblock In {\em Joint European Conference on Machine Learning and Knowledge
  Discovery in Databases}, pages 387--402. Springer, 2013.

\bibitem{buckman2018thermometer}
Jacob Buckman, Aurko Roy, Colin Raffel, and Ian Goodfellow.
\newblock Thermometer encoding: One hot way to resist adversarial examples.
\newblock In {\em International Conference on Learning Representations}, 2018.

\bibitem{carlini2017towards}
Nicholas Carlini and David Wagner.
\newblock Towards evaluating the robustness of neural networks.
\newblock In {\em IEEE Symposium on Security and Privacy}, pages 39--57. IEEE,
  2017.

\bibitem{carmon2019unlabeled}
Yair Carmon, Aditi Raghunathan, Ludwig Schmidt, John~C Duchi, and Percy~S
  Liang.
\newblock Unlabeled data improves adversarial robustness.
\newblock {\em Advances in Neural Information Processing Systems}, 32, 2019.

\bibitem{cisse2017parseval}
Moustapha Cisse, Piotr Bojanowski, Edouard Grave, Yann Dauphin, and Nicolas
  Usunier.
\newblock Parseval networks: Improving robustness to adversarial examples.
\newblock In {\em International Conference on Machine Learning}, pages
  854--863. PMLR, 2017.

\bibitem{coates2011analysis}
Adam Coates, Andrew Ng, and Honglak Lee.
\newblock An analysis of single-layer networks in unsupervised feature
  learning.
\newblock In {\em Proceedings of International Conference on Artificial
  Intelligence and Statistics}, pages 215--223. JMLR Workshop and Conference
  Proceedings, 2011.

\bibitem{cohen2019certified}
Jeremy Cohen, Elan Rosenfeld, and Zico Kolter.
\newblock Certified adversarial robustness via randomized smoothing.
\newblock In {\em International Conference on Machine Learning}, pages
  1310--1320. PMLR, 2019.

\bibitem{croce2020robustbench}
Francesco Croce, Maksym Andriushchenko, Vikash Sehwag, Edoardo Debenedetti,
  Nicolas Flammarion, Mung Chiang, Prateek Mittal, and Matthias Hein.
\newblock Robustbench: a standardized adversarial robustness benchmark.
\newblock {\em arXiv preprint arXiv:2010.09670}, 2020.

\bibitem{croce2020minimally}
Francesco Croce and Matthias Hein.
\newblock Minimally distorted adversarial examples with a fast adaptive
  boundary attack.
\newblock In {\em International Conference on Machine Learning}, pages
  2196--2205. PMLR, 2020.

\bibitem{croce2020reliable}
Francesco Croce and Matthias Hein.
\newblock Reliable evaluation of adversarial robustness with an ensemble of
  diverse parameter-free attacks.
\newblock In {\em International Conference on Machine Learning}, pages
  2206--2216. PMLR, 2020.

\bibitem{cui2021learnable}
Jiequan Cui, Shu Liu, Liwei Wang, and Jiaya Jia.
\newblock Learnable boundary guided adversarial training.
\newblock In {\em Proceedings of the IEEE International Conference on Computer
  Vision}, pages 15721--15730, 2021.

\bibitem{deng2009imagenet}
Jia Deng, Wei Dong, Richard Socher, Li-Jia Li, Kai Li, and Li Fei-Fei.
\newblock Imagenet: A large-scale hierarchical image database.
\newblock In {\em 2009 IEEE Conference on Computer Vision and Pattern
  Recognition}, pages 248--255, 2009.

\bibitem{Dhillon2018stochastic}
Guneet~S. Dhillon, Kamyar Azizzadenesheli, Jeremy~D. Bernstein, Jean Kossaifi,
  Aran Khanna, Zachary~C. Lipton, and Animashree Anandkumar.
\newblock Stochastic activation pruning for robust adversarial defense.
\newblock In {\em International Conference on Learning Representations}, 2018.

\bibitem{Goodfellow2015Explaining}
Ian Goodfellow, Jonathon Shlens, and Christian Szegedy.
\newblock Explaining and harnessing adversarial examples.
\newblock In {\em International Conference on Learning Representations}, 2015.

\bibitem{gowal2020uncovering}
Sven Gowal, Chongli Qin, Jonathan Uesato, Timothy Mann, and Pushmeet Kohli.
\newblock Uncovering the limits of adversarial training against norm-bounded
  adversarial examples.
\newblock {\em arXiv preprint arXiv:2010.03593}, 2020.

\bibitem{guo2018countering}
Chuan Guo, Mayank Rana, Moustapha Cisse, and Laurens van~der Maaten.
\newblock Countering adversarial images using input transformations.
\newblock In {\em International Conference on Learning Representations}, 2018.

\bibitem{he2020momentum}
Kaiming He, Haoqi Fan, Yuxin Wu, Saining Xie, and Ross Girshick.
\newblock Momentum contrast for unsupervised visual representation learning.
\newblock In {\em Proceedings of the IEEE/CVF Conference on Computer Vision and
  Pattern Recognition}, pages 9729--9738, 2020.

\bibitem{he2016deep}
Kaiming He, Xiangyu Zhang, Shaoqing Ren, and Jian Sun.
\newblock Deep residual learning for image recognition.
\newblock In {\em Proceedings of the IEEE Conference on Computer Vision and
  Pattern Recognition}, pages 770--778, 2016.

\bibitem{he2016identity}
Kaiming He, Xiangyu Zhang, Shaoqing Ren, and Jian Sun.
\newblock Identity mappings in deep residual networks.
\newblock In {\em European Conference on Computer Vision}, pages 630--645,
  2016.

\bibitem{Hein2017Formal}
Matthias Hein and Maksym Andriushchenko.
\newblock Formal guarantees on the robustness of a classifier against
  adversarial manipulation.
\newblock In I. Guyon, U.~V. Luxburg, S. Bengio, H. Wallach, R. Fergus, S.
  Vishwanathan, and R. Garnett, editors, {\em Advances in Neural Information
  Processing Systems 30}, pages 2263--2273. Curran Associates, Inc., 2017.

\bibitem{hendrycks2020augmix}
Dan Hendrycks, Norman Mu, Ekin~D. Cubuk, Barret Zoph, Justin Gilmer, and Balaji
  Lakshminarayanan.
\newblock {AugMix}: A simple data processing method to improve robustness and
  uncertainty.
\newblock {\em Proceedings of the International Conference on Learning
  Representations}, 2020.

\bibitem{ilyas2019adversarial}
Andrew Ilyas, Shibani Santurkar, Dimitris Tsipras, Logan Engstrom, Brandon
  Tran, and Aleksander Madry.
\newblock Adversarial examples are not bugs, they are features.
\newblock {\em Advances in Neural Information Processing systems}, 32, 2019.

\bibitem{krizhevsky2009learning}
Alex Krizhevsky.
\newblock Learning multiple layers of features from tiny images.
\newblock 2009.

\bibitem{Kurakin2017Adversarial}
Alexey Kurakin, Ian~J. Goodfellow, and Samy Bengio.
\newblock Adversarial machine learning at scale.
\newblock 2017.

\bibitem{lecun2015deep}
Yann LeCun, Yoshua Bengio, and Geoffrey Hinton.
\newblock Deep learning.
\newblock {\em Nature}, 521(7553):436--444, 2015.

\bibitem{madry2018towards}
Aleksander Madry, Aleksandar Makelov, Ludwig Schmidt, Dimitris Tsipras, and
  Adrian Vladu.
\newblock Towards deep learning models resistant to adversarial attacks.
\newblock In {\em International Conference on Learning Representations}, 2018.

\bibitem{mao2019metric}
Chengzhi Mao, Ziyuan Zhong, Junfeng Yang, Carl Vondrick, and Baishakhi Ray.
\newblock Metric learning for adversarial robustness.
\newblock {\em Advances in Neural Information Processing Systems}, 32, 2019.

\bibitem{netzer2011reading}
Yuval Netzer, Tao Wang, Adam Coates, Alessandro Bissacco, Bo Wu, and Andrew~Y
  Ng.
\newblock Reading digits in natural images with unsupervised feature learning.
\newblock In {\em NIPS Workshop on Deep Learning and Unsupervised Feature
  Learning}, 2011.

\bibitem{pang2021bag}
Tianyu Pang, Xiao Yang, Yinpeng Dong, Hang Su, and Jun Zhu.
\newblock Bag of tricks for adversarial training.
\newblock In {\em International Conference on Learning Representations}, 2021.

\bibitem{papernot2016distillation}
Nicolas Papernot, Patrick McDaniel, Xi Wu, Somesh Jha, and Ananthram Swami.
\newblock Distillation as a defense to adversarial perturbations against deep
  neural networks.
\newblock In {\em 2016 IEEE Symposium on Security and Privacy}, pages 582--597.
  IEEE, 2016.

\bibitem{rade2022reducing}
Rahul Rade and Seyed-Mohsen Moosavi-Dezfooli.
\newblock Reducing excessive margin to achieve a better accuracy vs. robustness
  trade-off.
\newblock In {\em International Conference on Learning Representations}, 2022.

\bibitem{rebuffi2021data}
Sylvestre-Alvise Rebuffi, Sven Gowal, Dan~Andrei Calian, Florian Stimberg,
  Olivia Wiles, and Timothy~A Mann.
\newblock Data augmentation can improve robustness.
\newblock {\em Advances in Neural Information Processing Systems},
  34:29935--29948, 2021.

\bibitem{rice2020overfitting}
Leslie Rice, Eric Wong, and Zico Kolter.
\newblock Overfitting in adversarially robust deep learning.
\newblock In {\em International Conference on Machine Learning}, pages
  8093--8104. PMLR, 2020.

\bibitem{samangouei2018defensegan}
Pouya Samangouei, Maya Kabkab, and Rama Chellappa.
\newblock Defense-{GAN}: Protecting classifiers against adversarial attacks
  using generative models.
\newblock In {\em International Conference on Learning Representations}, 2018.

\bibitem{schmidt2018adversarially}
Ludwig Schmidt, Shibani Santurkar, Dimitris Tsipras, Kunal Talwar, and
  Aleksander Madry.
\newblock Adversarially robust generalization requires more data.
\newblock {\em Advances in Neural Information Processing systems}, 31, 2018.

\bibitem{sehwag2019analyzing}
Vikash Sehwag, Arjun~Nitin Bhagoji, Liwei Song, Chawin Sitawarin, Daniel
  Cullina, Mung Chiang, and Prateek Mittal.
\newblock Analyzing the robustness of open-world machine learning.
\newblock In {\em Proceedings of the ACM Workshop on Artificial Intelligence
  and Security}, pages 105--116, 2019.

\bibitem{shamir2019simple}
Adi Shamir, Itay Safran, Eyal Ronen, and Orr Dunkelman.
\newblock A simple explanation for the existence of adversarial examples with
  small hamming distance.
\newblock {\em arXiv preprint arXiv:1901.10861}, 2019.

\bibitem{song2018pixeldefend}
Yang Song, Taesup Kim, Sebastian Nowozin, Stefano Ermon, and Nate Kushman.
\newblock Pixeldefend: Leveraging generative models to understand and defend
  against adversarial examples.
\newblock In {\em International Conference on Learning Representations}, 2018.

\bibitem{Szegedy2014Intriguing}
Christian Szegedy, Wojciech Zaremba, Ilya Sutskever, Joan Bruna, Dumitru Erhan,
  Ian Goodfellow, and Rob Fergus.
\newblock Intriguing properties of neural networks.
\newblock In {\em International Conference on Learning Representations}, 2014.

\bibitem{torralba200880}
Antonio Torralba, Rob Fergus, and William~T Freeman.
\newblock 80 million tiny images: A large data set for nonparametric object and
  scene recognition.
\newblock {\em IEEE transactions on pattern analysis and machine intelligence},
  30(11):1958--1970, 2008.

\bibitem{tramer2018ensemble}
Florian Tramèr, Alexey Kurakin, Nicolas Papernot, Ian Goodfellow, Dan Boneh,
  and Patrick McDaniel.
\newblock Ensemble adversarial training: Attacks and defenses.
\newblock In {\em International Conference on Learning Representations}, 2018.

\bibitem{vangansbeke2020scan}
Wouter Van~Gansbeke, Simon Vandenhende, Stamatios Georgoulis, Marc Proesmans,
  and Luc Van~Gool.
\newblock Scan: Learning to classify images without labels.
\newblock In {\em Proceedings of the European Conference on Computer Vision},
  2020.

\bibitem{Wang2020Improving}
Yisen Wang, Difan Zou, Jinfeng Yi, James Bailey, Xingjun Ma, and Quanquan Gu.
\newblock Improving adversarial robustness requires revisiting misclassified
  examples.
\newblock In {\em International Conference on Learning Representations}, 2020.

\bibitem{wu2020adversarial}
Dongxian Wu, Shu-Tao Xia, and Yisen Wang.
\newblock Adversarial weight perturbation helps robust generalization.
\newblock {\em Advances in Neural Information Processing Systems},
  33:2958--2969, 2020.

\bibitem{wu2018unsupervised}
Zhirong Wu, Yuanjun Xiong, Stella~X Yu, and Dahua Lin.
\newblock Unsupervised feature learning via non-parametric instance
  discrimination.
\newblock In {\em Proceedings of the IEEE Conference on Computer Vision and
  Pattern Recognition}, pages 3733--3742, 2018.

\bibitem{xie2018mitigating}
Cihang Xie, Jianyu Wang, Zhishuai Zhang, Zhou Ren, and Alan Yuille.
\newblock Mitigating adversarial effects through randomization.
\newblock In {\em International Conference on Learning Representations}, 2018.

\bibitem{yun2019cutmix}
Sangdoo Yun, Dongyoon Han, Seong~Joon Oh, Sanghyuk Chun, Junsuk Choe, and
  Youngjoon Yoo.
\newblock Cutmix: Regularization strategy to train strong classifiers with
  localizable features.
\newblock In {\em Proceedings of the IEEE International Conference on Computer
  Vision}, pages 6023--6032, 2019.

\bibitem{Zagoruyko2016Wide}
Sergey Zagoruyko and Nikos Komodakis.
\newblock Wide residual networks.
\newblock In {\em Proceedings of the British Machine Vision Conference}, 2016.

\bibitem{zhang2019you}
Dinghuai Zhang, Tianyuan Zhang, Yiping Lu, Zhanxing Zhu, and Bin Dong.
\newblock You only propagate once: Accelerating adversarial training via
  maximal principle.
\newblock {\em Advances in Neural Information Processing Systems}, 32, 2019.

\bibitem{zhang2018mixup}
Hongyi Zhang, Moustapha Cisse, Yann~N. Dauphin, and David Lopez-Paz.
\newblock mixup: Beyond empirical risk minimization.
\newblock In {\em International Conference on Learning Representations}, 2018.

\bibitem{zhang2019theoretically}
Hongyang Zhang, Yaodong Yu, Jiantao Jiao, Eric Xing, Laurent El~Ghaoui, and
  Michael Jordan.
\newblock Theoretically principled trade-off between robustness and accuracy.
\newblock In {\em International Conference on Machine Learning}, pages
  7472--7482. PMLR, 2019.

\bibitem{zhang2020attacks}
Jingfeng Zhang, Xilie Xu, Bo Han, Gang Niu, Lizhen Cui, Masashi Sugiyama, and
  Mohan Kankanhalli.
\newblock Attacks which do not kill training make adversarial learning
  stronger.
\newblock In {\em International Conference on Machine Learning}, pages
  11278--11287. PMLR, 2020.

\end{thebibliography}
}

\begin{appendices}

In this Supplementary Material, we provide more detailed analysis of experiments (Sec.~\ref{sec:1}), and the proof of the theorem (Sec.~\ref{sec:2}) in the main paper.

\section{More Experimental Analysis}
\label{sec:1}

In this section, we will conduct more extensive experiments to demonstrate the rationality and effectiveness of the proposed method. Specifically, in Sec.~\ref{sec:1.1}, we provide more ablation studies that are not included in the main paper due to the space limit. In Sec.~\ref{sec:1.2}, we verify the fidelity of the local structure mined during the adversarial training of our algorithm, as noted in Sec. 3.2 of the main paper.

\subsection{More Ablation Studies}
\label{sec:1.1}

\noindent
\textbf{Evaluation with Limited Training Data:} Previously, we have claimed that the main advantage of the proposed LSP method is to  alleviate the notorious reliance on the massive amount of data to achieve adversarial robustness. This is due to the structure prior of the interdependency relationship among data we have utilized to reduce the number of samples for learning the complicated underlying distribution. This is fundamentally different from most of the previous adversarial training methods, which seek more data~\cite{carmon2019unlabeled,schmidt2018adversarially}.

Here, to better demonstrate the advantage of the proposed method in reducing learning samples, we conduct experiments with less training data and compare our method with state-of-the-art adversarial training baseline methods. In Tab.~\ref{table:adv-small}, we use 10k-40k random subset of CIFAR-10 (with a total of 50k samples) to adversarially train all methods. With limited training data, our method significantly outperforms the comparison methods, which demonstrates the effectiveness of the proposed LSP term.

\begin{table*}[t]
\begin{center}
{\scriptsize
\begin{tabular}{l cccc cccc cccc cccc}
 \toprule
 \multicolumn{1}{c}{\textbf{Methods}} & \multicolumn{4}{c}{\textbf{CIFAR-10 (10000)}} & \multicolumn{4}{c}{\textbf{CIFAR-10 (20000)}} & \multicolumn{4}{c}{\textbf{CIFAR-10 (30000)}} & \multicolumn{4}{c}{\textbf{CIFAR-10 (40000)}} \\ \cmidrule(rl){2-5} \cmidrule(rl){6-9} \cmidrule(rl){10-13} \cmidrule(rl){14-17} 
 & Clean & CW & PGD & AA & Clean & CW & PGD & AA & Clean & CW & PGD & AA & Clean & CW & PGD & AA \\
 \toprule
  \textsc{AT}     & \textbf{62.81} & 30.12 & 30.74 & 27.26 & 66.34 & 38.50 & 39.58 & 35.31 & \textbf{74.40} & 42.17 & 44.27 & 39.17 & 77.46 & 45.59 & 46.70 & 42.18\\
 \textsc{TRADES}  & 61.18 & 32.69 & 34.77 & 31.62 & 67.56 & 35.28 & 37.96 & 34.13 & 73.29 & 39.61 & 42.70 & 38.02 & 77.62 & 43.88 & 47.00 & 42.86\\
 \textsc{MART}    & 55.08 & 34.33 & 37.50 & 33.10 & 61.12 & 37.00 & 40.86 & 35.36 & 68.02 & 44.09 & 46.49 & 40.98 & 73.13 & 46.93 & 48.41 & 42.13\\
 \textsc{HAT}     & 60.85 & 33.69 & 34.89 & 32.10 & \textbf{69.79} & 36.99 & 39.64 & 34.17 & 73.37 & 43.37 & 45.47 & 40.73 & \textbf{80.28} & 43.96 & 46.44 & 42.40\\
 \midrule
 \textsc{LSP+}    & 62.27 & \textbf{37.23} & \textbf{39.94} & \textbf{35.18} & 68.66 & \textbf{40.20} & \textbf{42.37} & \textbf{37.85} & 74.23 & \textbf{46.83} & \textbf{48.58} & \textbf{42.10} & 79.48 & \textbf{47.13} & \textbf{49.98} & \textbf{44.25}\\
 
 \bottomrule
\end{tabular}
}
\end{center}
\caption{The classification accuracy comparison of methods by Adversarial Training. We use four subsets of CIFAR-10 (10k, 20k, 30k, and 40k) as the training data. Numbers in bold indicate the best. Note that compared to our baseline \textsc{AT}, we achieve stronger robustness against all attacks.}
\label{table:adv-small}
\end{table*}

\noindent
\textbf{Neighborhood Area}
In Tab.~\ref{ablation:neighbor}, when we choose different numbers of nearest neighbors in our algorithm, the results show that more neighbors could benefit the performance, since more useful structure prior information is involved in the learning process. This implicitly supports the asymptotic theoretical conjecture in Sec. 3.3 of the main paper.

\begin{table}[htbp]
\begin{center}
{
\begin{tabular}{lcccccr}
\toprule
\# Neighbors & 2 & 4 & 8 & 16 & 32 \\
\midrule
Clean & 93.37 & 93.63 & \textbf{94.22} & 93.78 & 94.03 \\
CW & 54.86 & 54.26 & 62.53 & 63.25 & \textbf{64.45}\\
PGD & 12.94 & 14.81 & 28.00 & 29.99 & \textbf{30.14} \\
\bottomrule
\end{tabular}
}
\end{center}
\caption{Various choices of the number of local neighborhood.}
\label{ablation:neighbor}
\end{table}

\noindent
\textbf{Loss Weighting Parameter:}
Recall that the total objective function for training our LSP method is as follows:
\begin{equation}
    \mathcal{L}_{\text{Total}} = \frac{1}{|\mathcal{X}|}\sum_{x\in\mathcal{X}} (\mathcal{L}_{\text{CE}} + \lambda \mathcal{L}_{\text{LSP}}),
    \label{eq:total,non-adversarial}
\end{equation}
where $\lambda$ is a balancing parameter of the two loss terms. Here we study the effect of the weighting parameter $\lambda$ in Eqn.~\ref{eq:total,non-adversarial}. As shown in Tab.~\ref{ablation:weighting}, for all the attacks, a too-small $\lambda$ can not thoroughly benefit the adversarial robustness via the proposed LSP term; while a too-large $\lambda$ will be harmful to the data fitting which caused performance degradation.

\begin{table}[h]
\begin{center}
{\small
\begin{tabular}{lccccccr}
\toprule
$\lambda$ & 0.1 & 0.5 & 1 & 2 & 5 & 10 \\
\midrule
Clean & \textbf{94.45} & 94.16 & 94.22 & 92.80 & 90.06 & 88.64\\
CW & 36.91 & 55.59 & 62.53 & \textbf{65.93} & 55.39 & 52.18\\
PGD & 8.98 & 19.08 & \textbf{28.00} & 25.58 & 15.51 & 16.12\\
\bottomrule
\end{tabular}}
\caption{Various choices of loss weighting parameter $\lambda$.}
\label{ablation:weighting}
\end{center}
\end{table}

\begin{table}[htbp]
\begin{center}
{\footnotesize
\begin{tabular}{lcccccr}
\toprule
Methods & \textsc{Vanilla} & \textsc{Mixup} & \textsc{CutMix} & \textsc{AugMix} & \textbf{LSP}  \\
\midrule
PGD-$L_\infty$ & 6.26 & 12.23 & 13.14 & 16.17 & \textbf{28.00}  \\
PGD-$L_2$ & 9.73 & 18.62 & 22.38 & 19.73  & \textbf{34.30} \\
\bottomrule
\end{tabular}}
\caption{Evaluation with different perturbation norms of PGD attack, where the budget for $L_2$ norm is 64/255.}
\label{table:l2}
\end{center}
\end{table}

\noindent
\textbf{Different Perturbation Norms} Previously we evaluate with adversarial attacks based on $L_{\infty}$-norm. Here we also test the robustness against the PGD attack with $L_2$-norm in Tab.~\ref{table:l2}. Here, the maximum perturbation budget for $L_2$ attack is 64/255, and the  attacking step is 10 with the step size of 15/255. Again, our method shows stronger robustness in the metric norm of $L_2$ perturbation.


\subsection{Visulization of Local Neighbors}
\label{sec:1.2}

Furthermore, in the method of proposed LSP with adversarial training (cf. Sec. 3.2 of the main paper), to erase the concern about whether the local structure among training samples learned by the proposed method is valid, we show the visualization results of the nearest neighbors mined by our method during different training epochs in Figure~\ref{fig:visualization}. For both CIFAR-10 and SVHN datasets, the neighbors of the anchor samples are random initially, but the semantic relationship grows clearly with the training epochs, indicating the local structure is improving to benefit the learning. In the rightmost of Figure~\ref{fig:visualization}, we also show the evolving curve of purity, which is defined as the average fraction of the neighbors that share the same semantic label with the anchor in all neighbors. For both datasets, the purity grows to over 90\%, implying the validity of the learned local structure.

\begin{figure*}[htbp]
    \centering
    \includegraphics[width=\linewidth]{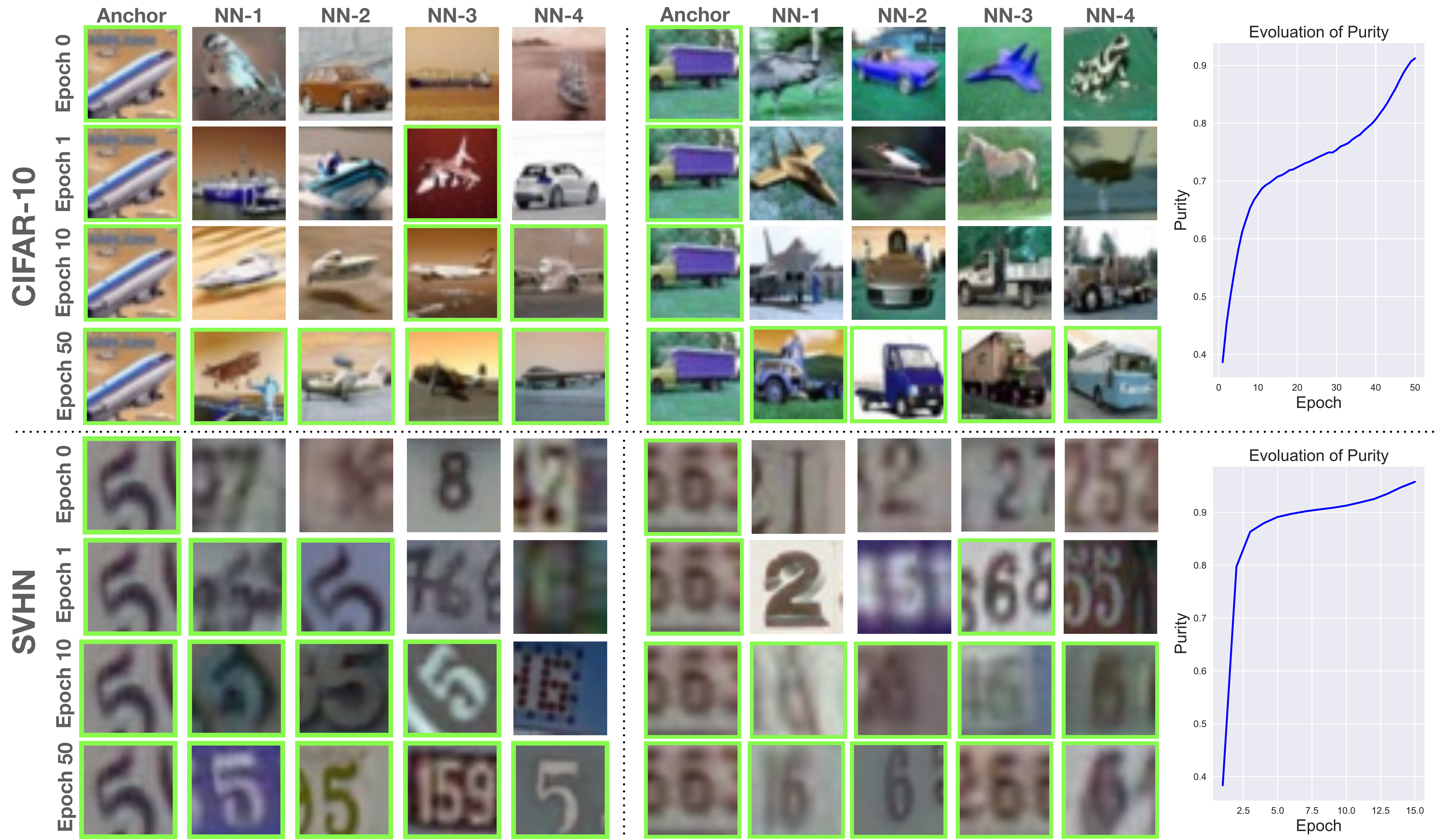}
    \caption{The visualization results of the mined local structure in the proposed LSP with adversarial training, for the dataset of CIFAR-10 (top) and SVHN (bottom). The 1-4 nearest-neighbors of the anchor are depicted, with the \textcolor{green}{green} frame indicating the same semantic label sharing with the anchor. It is clear that during training, the semantic similarity of the local neighbors with the anchor is increasing, also evidenced by the purity curve on the right.}
    \label{fig:visualization}
\end{figure*}

\section{Proof of Theoretical Analysis} \label{sec:2}

Recall the theorem in the main paper as follows. Here we give the proof of the theorem.

\begin{theorem} \label{theorem}
Suppose $a,b\in\mathcal{Y}$ are the most and second-most likely class prediction of the classifier $f$, i.e., 
\begin{equation}
    a = \arg\max_{l\in\mathcal{Y}} f(x)_l, \quad  
    b = \arg\max_{l\in\mathcal{Y}\backslash\{a\}} f(x)_l,
\end{equation}
and $p_a, p_b$ are the predictive probability of the classifier w.r.t. label $a$ and $b$, i.e.,
\begin{equation}
    p_a = f(x)_a, \quad p_b = f(x)_b.
\end{equation}
If function $f$ is locally $L_x$-Lipschitiz on input $x\in \mathcal{X}$, i.e., 
\begin{align}
    & \forall x\in \mathcal{X}, \forall x^\prime, \text{s.t. } ||x^\prime-x||\le\delta, \quad \\
    &||f(x) - f(x^\prime)|| \le L_x ||x-x^\prime||.
    \label{eq:theorem}
\end{align}
Then $f$ is guaranteed to be robust at $x$ up to any perturbation of magnitude $\delta \le \frac{1}{2L_x}(p_a-p_b)$.
\end{theorem}

\begin{proof}
Given $x$, for $\forall y\in \mathcal{Y}$, applying Lipschitz condition of $f$, if $x^\prime$ lies in the $\delta$-neighborhood of $x$, we have:
\begin{equation}
    |f(x^\prime)_y-f(x)_y| \le L_x ||x^\prime - x|| \le L_x \delta \le \frac{1}{2}(p_a - p_b).
\end{equation}
Assign $y=a$, we have:
\begin{align}
    f(x^\prime)_a &\ge f(x)_a - |f(x^\prime)_a-f(x)_a| \notag\\
    & \ge p_a - \frac{1}{2}(p_a - p_b) = \frac{p_a+p_b}{2}.
\end{align}
Assign $y$ as $\forall l\ne a$, we have:
\begin{align}
    f(x^\prime)_l & \le f(x)_l + |f(x^\prime)_l-f(x)_l| \notag\\
    & \le p_b + \frac{1}{2}(p_a-p_b) = \frac{p_a+p_b}{2}
\end{align}
Thus,
\begin{equation}
    \arg\max_{y\in \mathcal{Y}} f(x^\prime)_y = a.
\end{equation}
\end{proof}

From our theoretical analysis in Sec 3.3, our proposed Local Structure Preserving (LSP) is motivated by the goal of bounding the Lipschitz constant of the learned network. Whereas the relationship between bounding the Lipschitz constant and enhancing robustness is broadly-known and well-studied. This is the insight that our LSP term helps enhance robustness.

From another perspective, the ERM-based learning framework treats each training sample as \emph{i.i.d.}, thus may inevitably introduce the ``pockets'' appearing in the data manifold, which causes the phenomenon of adversarial vulnerability~\cite{Szegedy2014Intriguing}. While our method treats each training sample as \emph{non-i.i.d.} and considers preserving the local manifold around each training sample, which will intuitively reduce the possibility of such ``pockets'' thus enhancing the robustness.

\end{appendices}

\end{document}